%% file: main.tex
\documentclass[twoside,11pt]{article}

\usepackage{jmlr2e}

\usepackage{xcolor}
\definecolor{citecol}{HTML}{6F130C}
\definecolor{tableofcontent}{HTML}{1F4A83}
\definecolor{urlcol}{HTML}{2470D8}

\usepackage{hyperref}
\hypersetup{
    colorlinks=true,       
    linkcolor=tableofcontent,          
    citecolor=citecol,        
    urlcolor=blue,           
}
\usepackage{microtype}
\usepackage{graphicx}
\usepackage{subfigure}
\usepackage{booktabs} 
\input{math_commands}

\newcounter{bxincomm}
\definecolor{aqua}{rgb}{0.00,0.67,0.80}

\newcounter{ygcounter}

\newcommand{\ygc}[1]{\ygc{\stepcounter{ygcounter}{\bf [YG's comment \arabic{ygcounter}: #1]}\;}}

\jmlrheading{1}{2023}{1-48}{4/00}{10/00}{Xinliang Liu, Bingxin Zhou, Chutian Zhang, Yu Guang Wang}

\ShortHeadings{Framelet Message Passing}{Liu, Zhou, Zhang, Wang}
\firstpageno{1}

\begin{document}

\title{Framelet Message Passing}

\author{
\name Xinliang Liu \email xinliang.liu@kaust.edu.sa \\
\addr King Abdullah University of Science and Technology\\
Thuwal 23955, Saudi Arabia
\AND
\name Bingxin Zhou \email bingxin.zhou@sjtu.edu.cn \\
\addr Institute of Natural Sciences and Shanghai National Center for Applied Mathematics (SJTU Center)\\ Shanghai Jiao Tong University, Shanghai 200240, China.
\AND
\name Chutian Zhang \email scarborough@sjtu.edu.cn \\
\addr Institute of Natural Sciences and School of Mathematical Sciences\\ Shanghai Jiao Tong University, Shanghai 200240, China.
\AND
\name Yu Guang Wang \email yuguang.wang@sjtu.edu.cn \\
\addr Institute of Natural Sciences and School of Mathematical Sciences\\ Shanghai Jiao Tong University, Shanghai 200240, China.
}

\editor{NA}

\maketitle

\begin{abstract}
Graph neural networks (GNNs) have achieved champion in wide applications. Neural message passing is a typical key module for feature propagation by aggregating neighboring features. In this work, we propose a new message passing based on multiscale framelet transforms, called Framelet Message Passing. Different from traditional spatial methods, it integrates framelet representation of neighbor nodes from multiple hops away in node message update. We also propose a continuous message passing using neural ODE solvers. It turns both discrete and continuous cases can provably achieve network stability and limit oversmoothing due to the multiscale property of framelets. Numerical experiments on real graph datasets show that the continuous version of the framelet message passing significantly outperforms existing methods when learning heterogeneous graphs and achieves state-of-the-art performance on classic node classification tasks with low computational costs.
\end{abstract}


\tableofcontents

\section{Introduction}
Graph neural networks (GNNs) have received growing attention in the past few years \citep{bronstein2017geometric,hamilton2020graph,wu2020comprehensive}. The key to successful GNNs is the equipment of effective graph convolutions that distill useful features and structural information of given graph signals. Existing designs on graph convolutions usually summarize a node's local properties from its spatially-connected neighbors. Such a scheme is called message passing \citep{gilmer2017neural}, where different methods differentiate each other by their unique design of the aggregator \citep{kipf2016semi,hamilton2017inductive,velivckovic2017graph}. 
Nevertheless, spatial convolutions are usually built upon the first-order approximation of eigendecomposition by the graph Laplacian, and they are proved recklessly removing high-pass information in the graph \citep{wu2019simplifying,oono2019graph,bo2021beyond}. Consequently, many local details are lost during the forward propagation. The information loss becomes increasingly ineluctable along with the raised number of layers or the expanded range of neighborhoods. 
This deficiency limits the expressivity of GNNs and partly gives rise to the oversmoothing issue of a deep GNN. Alternatively, a few existing spectral-involved message passing schemes use eigenvectors to feed the projected node features into the aggregator \citep{stachenfeld2020graph,balcilar2021breaking,beaini2021directional}. While the eigenvectors capture the directional flow in the input by Fourier transforms, they overlook the power of multi-scale representation, which is essential to preserve sufficient information in different levels of detail. Consequently, neither the vanilla spectral graph convolution nor eigenvector-based message passing is capable of learning stable and energy-preserving representations.

\begin{figure}[t]
    \centering
    \includegraphics[width=\linewidth]{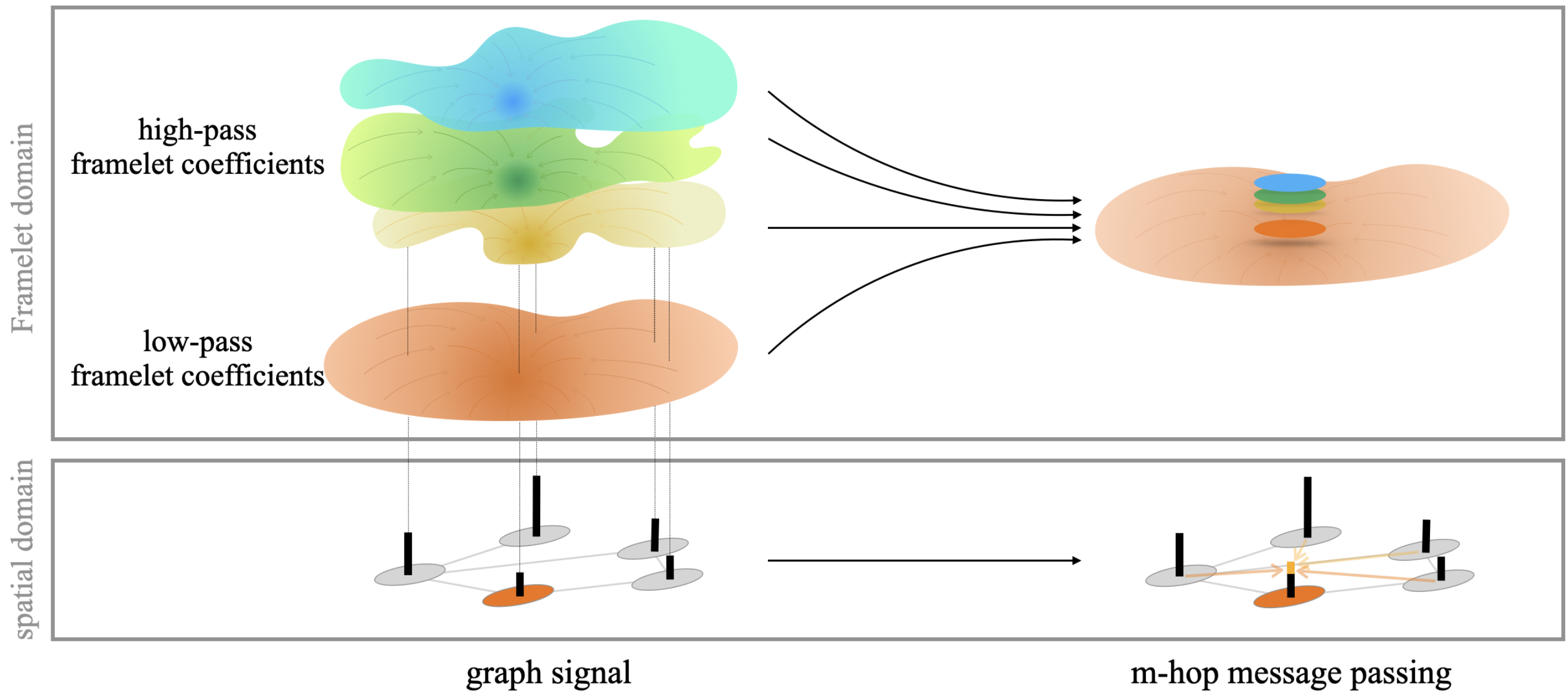}
    \caption{An illustrative workflow of the proposed framelet message passing. An input graph signal is first decomposed into multi-scale coefficients (colored polygons) in the framelet domain. In a convolution layer, each of the framelet coefficients aggregates m-hop neighbors' coefficients from the same level and scale to update its representation. The multi-scale node representation is then summed up as the propagated new representation of the node. In comparison, multiple spatial-based graph convolution layers are required to accelerate the same range of node information, and they are poor at memorizing long-range information.}
    \label{fig:architecture}
\end{figure}

To tackle the issue of separately constructing spectral graph convolutions or spatial message passing rules, this work establishes a spectral message passing scheme with multiscale graph framelet transforms. For a given graph, framelet decomposition generates a set of framelet coefficients in the spectral domain with low-pass and high-pass filters. The coefficients with respect to individual nodes then follow the message passing scheme \citep{gilmer2017neural} to integrate their neighborhood information from the same level. The proposed framelet message passing (FMP) has shown promising theoretical properties. 

First, \textbf{it can limit oversmoothing with a non-decay Dirichlet energy during propagation}, and the Dirichlet energy would not explode when the network goes deep.
Unlike the conventional message passing convolutions that have to repeat $m$ times to cover a relatively large range of $m$-hop neighbors for a central node, the framelet message passing reaches out all $m$-hop neighbors in a single graph convolutional layer. Instead of cutting down the influence of distant neighboring nodes by the gap to the central node, the framelet way  disperses small and large ranges of neighboring communities in different scales and levels, where in low-pass framelet coefficients the local approximated information is retained, and high-pass coefficients mainly hold local detailed information. On each of the scales when conducting message passing, the $m$-hop information is adaptively accumulated to the central node, which can be considered as an analog to the graph rewiring, and it is helpful for circumventing the long-standing oversmoothing issue in graph representation learning. 

Meanwhile, \textbf{FMP is stable on perturbed node features}. Transforming the graph signal from the spatial domain to the framelet domain divides uncertainties from the corrupted input signal. Through controlling the variance within an acceptable range in separate scales, we prove that the processed representation steadily roams within a range, \textit{i.e.,} the framelet message passing is stable to small input perturbations.

In addition, \textbf{FMP bypasses unnecessary spectral transforms and improves convolutional efficiency}. On top of approaching long-range neighbors in one layer, FMP also avoids the inverse framelet transforms in the traditional framelet convolution \citep{zheng2021framelets} and integrates the low and high pass features by adjusting their feature-wise learnable weights in the aggregation.

The rest of the paper starts by introducing the message passing framework in Section~\ref{sec:preliminary}, and discussing its weakness in the oversmoothing issue in Section~\ref{sec:oversmoothing}. Based on the established graph Framelet system and its favorable properties (Section~\ref{sec:ufs}). Based on this, Section~\ref{sec:fmp} gives the two variants of the proposed FMP, whose energy-preserving effect and stability are justified in Section~\ref{sec:fmp_oversmoothing} and \ref{sec:fmp_stability}, respectively. The empirical performances of FMP are reported in Section~\ref{sec:exp} for node classification tasks on homogeneous and heterogeneous graphs, where FMP achieves state-of-the-art performance. We review the previous literature in the community in Section~\ref{sec:review}, and then conclude the work in Section~\ref{sec:conclusion}.

\section{Graph Representation Learning with Message Passing}
\label{sec:preliminary}
An undirected attributed graph $\gG=(\gV,\gE,\mX)$ consists of a non-empty finite set of $n=|\gV|$ nodes $\gV$ and a set of edges $\gE$ between node pairs. Denote $\mA\in\R^{N\times N}$ the (weighted) graph adjacency matrix and $\mX\in\R^{N\times d}$ the node attributes. A graph convolution learns a matrix representation $\mH$ that embeds the structure $\mA$ and feature matrix $\mX=\{\mX_{j}\}_{j=1}^N$ with $\mX_{j}$ for node $j$.

Message passing \citep{gilmer2017neural} defines a general framework of feature propagation rules on a graph, which updates a central node's smooth representation by aggregated information (\textit{e.g.}, node attributes) from connected neighbors. At a specific layer $t$, the propagation for the $i$th node reads
\begin{equation} \label{eq:mp}
\begin{aligned}
    \mX_i^{(t)} &= \gamma\left(\mX^{(t-1)}_i, \mZ^{(t)}\right)\\
    \mZ^{(t)} &= \square_{j\in \mathcal{N}(i)} \phi(\mX^{(t-1)}_i,\mX^{(t-1)}_j,\mA_{ij}),
\end{aligned}
\end{equation}
where $\square(\cdot)$ is a differentiable and permutation invariant aggregation function, such as summation, average, or maximization. Next, the aggregated representation of neighbor nodes $\mZ^{(t)}$ is used to update the central node's representation, where two example operations are addition and concatenation. Both $\gamma(\cdot)$ and $\phi(\cdot)$ are differentiable aggregation functions, such as MLPs. The node set $\gN(i)$ includes $\gV_i$ and other nodes that are connected directly with $\gV_i$ by an edge, which we call node $\gV_i$'s 1-hop neighbors.

The majority of (spatial-based) graph convolutional layers are designed following the message passing scheme when updating the node representation. For instance, GCN \citep{kipf2016semi} adds up the degree-normalized node attributes from neighbors (including itself) and defines
\begin{equation*} 
    \mX_i^{(t)} = \sigma \left(\sum_{j\in \mathcal{N}(i)\cup\{i\}}\frac1{\sqrt{1+ d_i}\sqrt{1 +d_j}}\mX_j^{(t-1)}\mW \right),
\end{equation*}
where $\mW$ is learnable weights, $d_i= \sum_{j\in \mathcal{N}(i)}\mA_{i,j}$ and $\mD={\rm diag}(d_1,\dots,d_N)$ is the degree matrix for $\mA$. Instead of the pre-defined adjacency matrix, GAT \citep{velivckovic2017graph} aggregates neighborhood attributes by learnable attention scores and \textsc{GraphSage} \citep{hamilton2017inductive} averages the contribution from the sampled neighborhood. Alternatively, GIN \citep{xu2019powerful} attaches a custom number of MLP layers for $\square(\cdot)$ after the vanilla summation. While these graph convolutions construct different formulations, they merely make a combination of inner product, transpose, and diagonalization operations on the graph adjacency matrix, which fails to distinguish different adjacency matrice by the 1-WL test \citep{balcilar2021breaking}. In contrast, spectral-based graph convolutions require eigenvalues or eigenvectors to construct the update rule and create expressive node representations towards the theoretical limit of the 3-WL test. 

On top of the expressivity issue, spectral-based convolutions have also proven to ease the stability concern that is widely observed in conventional spatial-based methods. To circumvent the two identified problems, we propose a spectral-based message passing scheme for graph convolution, which is stable and has the ability to alleviate the oversmoothing issues.

\section{Depth Limitation of Message Passing by Oversmoothing}
\label{sec:oversmoothing}
The depth limitation prevents the performance of many deep GNN models. The problem was first identified by \cite{li2018deeper}, where many popular spatial graph convolutions apply Laplacian smoothing to graph embedding. Shallow GNNs perform global denoising to exclude local perturbations and achieve state-of-the-art performance in many semi-supervised learning tasks. While a small number of graph convolutions has limited expressivity, deeply stacking the layers leads the connected nodes to converge to indistinguishable embeddings. Such an issue is widely known as \emph{oversmoothing}.

One way to understand the oversmoothing issue is through the Dirichlet energy, which measures the average distance between connected nodes in the feature space. The graph Laplacian of $\gG$ is defined by $\gL = \mD - \mA$.  Let $\tilde{\mA}:=\mA+\mI_N, \tilde{\mD}:=\mD+\mI_N$ be the adjacent and degree matrix of graph $\gG$ augmented with self-loops and the normalized graph Laplacian is defined by $\widetilde{\gL}:=\widetilde{\mD}^{-1/2}\gL\widetilde{\mD}^{-1/2}$.  Formally, the Dirichlet energy of a node feature $\mX$ from $\gG$ with normalized $\widetilde{\gL}$ is defined by  
\begin{equation*}
\begin{aligned}
    E(\mX) &= \tr(\mX^\top\widetilde{\gL}\mX)= \frac{1}{2}\sum \mA_{ij}\left(\frac{\mX_i}{\sqrt{1+d_i}}-\frac{\mX_j}{\sqrt{1+d_j}}\right)^2.
\end{aligned}
\end{equation*}

The energy evolution provides a direct indicator for the degree of feature expressivity in the hidden space. For instance, \cite{cai2020note} observed that GCN \citep{kipf2016semi} has the Dirichlet energy decaying rapidly to zero as the network depth increases, which indicates the local high-frequency signals are ignored during propagation. It is thus desired that the Dirichlet energy of the encoded features is bounded for deep GNNs.

We start with GCN as an example to illustrate the cause of oversmoothing. Set $\mP:=\mI_N-\widetilde{\gL}$. It is observed in \cite{oono2019graph} that a multi-layer GCN simply writes in the form  $f=f_L \circ \cdots \circ f_1$ where $f_l: \mathbb{R}^{N \times d} \rightarrow \mathbb{R}^{N \times d}$ is defined by $f_l(\mX):=\sigma\left(\mP \mX \mW_l\right)$.  Although the asymptotic behavior of the output $X^{(L)}$ of the $\mathrm{GCN}$ as $L \rightarrow \infty$ is investigated in \cite{oono2019graph} with oversmoothing property, to facilitate the analysis for our model MPNN, we also sketch the proof for completeness here.

\begin{lemma}[\cite{oono2019graph}]
\label{lem:activation}
For any $\mX \in \mathbb{R}^{n\times d}$,  we have 
\begin{enumerate}
    \item[1)]  $E(\mX \mW_l)\leq \mu_{\max}^2 E(\mX)  $, where $\mu_{\max} $ denotes the  singular value for $\mW_l$  with the largest absolute value.
    \item[2)] $ E( \sigma(\mX) ) \leq  E(\mX) $.
\end{enumerate}
\end{lemma}
\begin{lemma}[\cite{oono2019graph}]
      Without loss of generality, we suppose $\gG$  is connected. Let $\lambda_1 \leq \cdots \leq \lambda_n$ be the eigenvalue of $\mP$ sorted in ascending order. Then, we have
      \begin{enumerate}
          \item[1)] $-1<\lambda_1, \lambda_{n-1}<1$, and $\lambda_{n}=1$, hence $\lambda_{\max}:= \max_{i=1}^{n-1}\left|\lambda_i\right|<1$.
          \item[2)] $ E(\mP\mX) \leq \lambda_{\max}^2 E(\mX) $. 
      \end{enumerate}  
\end{lemma}
\begin{proof}
Note that $\mP$ and $\widetilde{\gL}$ share the same eigenspace and the Dirichlet energy is closely associated with the spectrum of $\mP$ and $\widetilde{\gL}$. Recall that $\mP = \mI_N-\widetilde{\gL} = \widetilde{\mD}^{-1/2}\widetilde{\mA}\widetilde{\mD}^{-1/2}$. Since the matrix similarity between $\mP$ and $\widetilde{\mD}^{-1} \widetilde{\mA}$, it suffices to investigate the spectrum of $ \widetilde{\mD}^{-1} \widetilde{\mA} $. Let $\tilde{\lambda}_1 \leq \cdots \leq \tilde{\lambda}_n$ be the eigenvalue of $\mP$ sorted in ascending order. It can be verified that $\widetilde{\mD}^{-1} \widetilde{\mA}$ is a stochastic matrix; hence, by Perron–Frobenius theorem, we have $-1<\tilde{\lambda}_1, \leq \tilde{\lambda}_{n-1}<1$ and $\tilde{\lambda_n}=1$ and $\mathbf{1}$  is the only eigenvector corresponding to $\lambda_{n}=1$. Then, we conclude the first claim. Moreover, we obtain $\widetilde{\gL}$ is symmetric positive semi-definite matrix with eigenvalues $0=\lambda^{\gL}_1 <\lambda^{\gL}_2<\cdots < \lambda^{\gL}_{n}<2$ and $\mathbf{1}$  is the only eigenvector corresponding to $\lambda^{\gL}_1=0$. Combing the fact that $E(\mX) = \tr(\mX^\top\widetilde{\gL}\mX)$ we conclude $E(\mP\mX) \leq \lambda_{\max}^2 E(\mX) $. It means the graph convolution contracts the energy by a factor of $\lambda_{\max}^2$, and we also get a by-product that the convolution shrinks the feature $\mX$ except for the constant component.
\end{proof}  

By the above Lemmas, we obtain the \emph{oversmoothing} property of GCN as 
\begin{theorem}[\cite{oono2019graph}] Let ${\rm GCN}_L:=f_L\circ f_{L-1}\circ \cdots \circ f_1$ be a graph convolutional network with $L$ layers with input feature $X$. Then, the Dirichlet energy of ${\rm GCN}_L$ is bounded by
    \begin{equation}
        E({\rm GCN}_L(\mX)) < \left(\mu_{\max} \lambda_{\max}\right)^{2L}E(\mX).
    \end{equation}
Suppose the $\mu_{\max} \lambda_{\max} < 1$, then the energy decays  exponentially with layers.
\end{theorem}

\section{Undecimated Framelet System}
\label{sec:ufs}
The $\{\vg_\ell\}_{\ell=1}^M$ from $l_2(\gG)$ is said a \emph{frame} for $l_2(\gG)$ is a collection of elements  if there exist constants $A$ and $B$, $0<A\le B< \infty$, such that
\begin{equation}
\label{defn:frame}
    A\|\vf\|^2 \le \sum_{\ell=1}^M |\ipG{\vf,\vg_\ell}|^2 \le B\|\vf\|^2\quad \forall \vf \in l_2(\gG).
\end{equation}
Here $A, B$ are called \emph{frame bounds}. When  $A=B=1$, $\{\vg_\ell\}_{\ell=1}^M$ is said to form a \emph{tight frame} for $l_2(\gG)$. In this case, \eqref{defn:frame} is alternatively written as
\begin{equation}
\label{defn:frame2}
    \vf = \sum_{\ell=1}^M \ipG{\vf,\vg_\ell}\vg_\ell,
\end{equation}
which follows from the polarization identity. 
For a tight frame $\{\vg_\ell\}_{\ell=1}^M$ with $\|\vg_\ell\|=1$ for $\ell=1,\dots,M$, there must holds $M=N$, and $\{\vg_\ell\}_{\ell=1}^N$ forms an orthonormal basis for $l_2(\gG)$. Tight frames ensures the one-to-one mapping between framelet coefficients $\ipG{\vf,\vg_\ell}$ and the original vector $\vf$ \citep{Daubechies1992}.

Let $\Psi=\{\alpha;\beta^{(1)},\dots,\beta^{(K)}\}$ be a set of functions in $L_1(\R)$, where $L_1(\R)$ refers to the functions that are absolutely integrable on $\R$ with respect to the Lebesgure measure. The \emph{Fourier transform} $\FT{\gamma}$ of a function $\gamma\in L_1(\R)$ is defined by $\FT{\gamma}(\xi):=\int_{\R}\gamma(t)e^{-2\pi i t\xi}\: \mathrm{d} t$, where $\xi\in\R$. The Fourier transform can be extended from $L_1(\R)$ to $L_2(\R)$, which is the space of square-integrable functions on $\R$. See \cite{stein2011fourier} for further information.

A \emph{filter bank} is a set of filters, where a \emph{filter (or mask)} $\mask:=\{\mask_k\}_{k\in\sZ}\subseteq \sC$ is a complex-valued sequence in $l_1(\sZ):=\{h=\{h_k\}_{k\in\sZ}\subseteq\sC : \sum_{k\in\sZ} |h_k|<\infty \}$. The \emph{Fourier series} of a sequence $\{\mask_k\}_{k\in\sZ}$ is the $1$-periodic function $\FT{\mask}(\xi):=\sum_{k\in\sZ}\mask_k e^{-2\pi i k\xi}$ with $\xi\in\R$. Let $\Psi=\{\alpha;\beta^{(1)},\dots,\beta^{(K)}\}$ be a set of \emph{framelet generators} associated with a filter bank $\boldsymbol{\eta}:=\{\maska; \maskb[1],\ldots,\maskb[K]\}$. Then the Fourier transforms of the functions in $\Psi$ and the filters' corresponding Fourier series in $\boldsymbol{\eta}$ satisfy
\begin{equation}
\label{eq:refinement}
    \FT{\alpha}(2\xi) = \FS{\maska}(\xi)\FT{\alpha}(\xi),\quad
    \FT{\beta^{(r)}}(2\xi) = \FS{\maskb}(\xi)\FT{\alpha}(\xi),\quad r=1,\ldots,K, \; \xi\in\R.
\end{equation}

We give two typical examples of filters and scaling functions, as follows.

\paragraph{Example.1} The first one is the Haar-type filters with one high pass: for $\xi \in \mathbb{R}$,
\begin{equation}
   \label{eqs:mask.numer.s4}
    \widehat{a}(\xi)=\cos (\xi / 2),\quad  \widehat{b^{(1)}}(\xi)=\sin (\xi / 2)
\end{equation}
with scaling functions
\begin{equation*}
    \FT{\alpha}(\xi) = \frac{\sin(\xi/2)}{\xi/2},\quad 
    \FT{\beta}(\xi) = \sqrt{1-\left(\frac{\sin(\xi/2)}{\xi/2}\right)^2}.
\end{equation*}


\paragraph{Example.2} Another example of filters and scaling functions with two high passes are from \cite[Chapter~4]{Daubechies1992}:
\begin{subequations}\label{eqs:mask.numer.s3}
\begin{align}
  \FT{\maska}(\xi)&: =
  \left\{\begin{array}{ll}
    1, & |\xi|<\frac{1}{8},\\[1mm]
    \cos\bigl(\frac{\pi}{2}\hspace{0.3mm} \nu(8|\xi|-1)\bigr), & \frac{1}{8} \le |\xi| \le \frac{1}{4},\\[1mm]
    0, & \frac14<|\xi|\le\frac12,
    \end{array}\right. \\[1mm]
  \FT{\maskb[1]}(\xi)&:  =\left\{\begin{array}{ll}
    0, & |\xi|<\frac{1}{8},\\[1mm]
    \sin\bigl(\frac{\pi}{2}\hspace{0.3mm} \nu(8|\xi|-1)\bigr), & \frac{1}{8} \le |\xi| \le \frac{1}{4},\\[1mm]
    \cos\bigl(\frac{\pi}{2}\hspace{0.3mm} \nu(4|\xi|-1)\bigr), & \frac14<|\xi|\le\frac12.
    \end{array}\right.\\[1mm]
  \FT{\maskb[2]}(\xi)&:
  =\left\{\begin{array}{ll}
    0, & |\xi|<\frac{1}{4},\\[1mm]
    \sin\bigl(\frac{\pi}{2}\hspace{0.3mm} \nu(4|\xi|-1)\bigr), & \frac{1}{4} \le |\xi| \le \frac{1}{2},
    \end{array}\right.
\end{align}
\end{subequations}
where
\begin{equation*}
  \nu(t) := t^{4}(35 - 84t + 70t^{2} - 20 t^{3}), \quad t\in \mathbb{R}.
\end{equation*}
The associated framelet generators $\Psi=\{\scala; \scalb^1,\scalb^2\}$ are defined by
\begin{subequations}\label{eqs:scal.numer.s3}
\begin{align}
  \FT{\scala}(\xi)&=
  \left\{\begin{array}{ll}
    1, & |\xi|<\frac{1}{4},\\[1mm]
    \cos\bigl(\frac{\pi}{2}\hspace{0.3mm} \nu(4|\xi|-1)\bigr), & \frac{1}{4} \le |\xi| \le \frac{1}{2},\\[1mm]
    0, & \hbox{else},
    \end{array}\right. \\[1mm]
  \FT{\scalb^1}(\xi)&= \left\{\begin{array}{ll}
    \sin\left(\frac{\pi}{2}\hspace{0.3mm} \nu(4|\xi|-1)\right), & \frac{1}{4}\le|\xi|<\frac{1}{2},\\[1mm]
    \cos^2\left(\frac{\pi}{2}\hspace{0.3mm} \nu(2|\xi|-1)\right), & \frac{1}{2} \le |\xi| \le 1,\\[1mm]
    0, & \hbox{else},
    \end{array}\right.\\[1mm]
  \FT{\scalb^2}(\xi)&= \left\{\begin{array}{ll}
   0, &|\xi|<\frac{1}{2},\\[1mm]
    \cos\left(\frac{\pi}{2}\hspace{0.3mm} \nu(2|\xi|-1)\right) \sin\left(\frac{\pi}{2}\hspace{0.3mm} \nu(2|\xi|-1)\right), & \frac{1}{2} \le |\xi| \le 1,\\[1mm]
    0, & \hbox{else}.
    \end{array}\right.
\end{align}
\end{subequations}

\begin{figure}[!tp]
    \centering
    \includegraphics[width=0.47\linewidth]{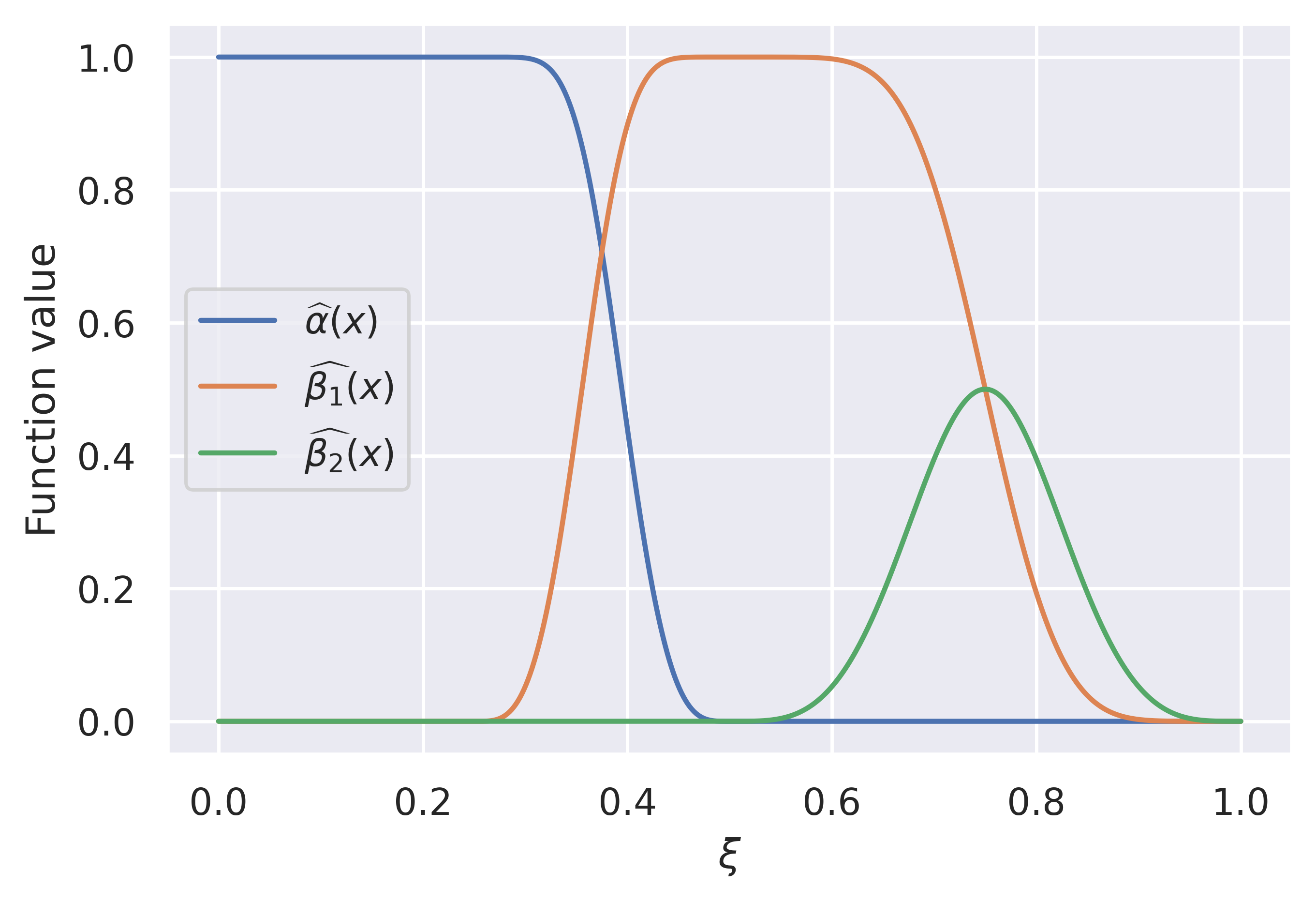}
    \hspace{2mm}
    \includegraphics[width=0.47\linewidth]{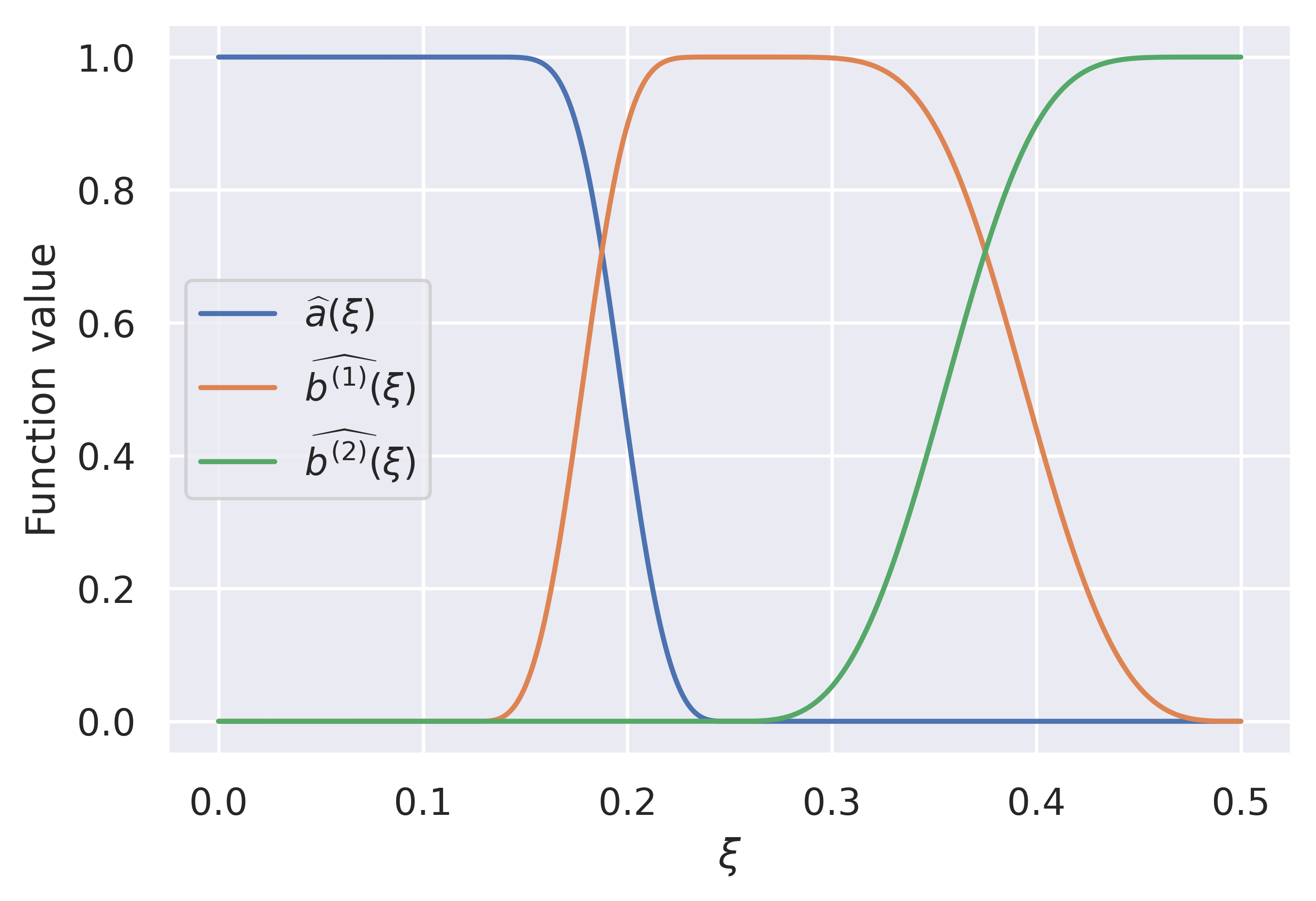}
    \caption{Filters and Scaling functions with two high passes in \eqref{eqs:mask.numer.s3} and \eqref{eqs:scal.numer.s3}.} 
    \label{fig:filter}
\end{figure}

\begin{definition}[Undecimated Framelet System \citep{dong2017sparse,zheng2022decimated}]
Given a filter bank $\boldsymbol{\eta}:=\{a;b^{(1)},\dots,b^{(K)}\}$ and scaling functions $\Psi=\{\alpha;\beta^{(1)},\dots,\beta^{(K)}\}$, an undecimated framelet system $\operatorname{UFS}_{J_1}^{J}(\Psi, \boldsymbol{\eta})$ ($J > J_1$) for $l_2(\gG)$ from $J_1$ to $J$ is defined by
\begin{equation}
\begin{aligned}
\label{defn:UFS}
  \operatorname{UFS}_{J_1}^{J}(\Psi, \boldsymbol{\eta}) &:=\operatorname{UFS}_{J_1}^{J}(\Psi, \boldsymbol{\eta}; \gG) \\
  &:=\left\{\varphi_{J_1, p}: p \in \gV \right\} \cup\{\psi_{l,p}^{(k)}: p \in \gV, l=J_1,\dots,J\}_{k=1}^{K}. \notag
\end{aligned}
\end{equation}
\end{definition}
In this paper, we focus on undecimated framelets, which maintain a constant number of framelets at each level of the decomposition. Decimated framelet systems, on the other hand, can be created by constructing a coarse-grained chain for the graph, as described in detail in \cite{zheng2022decimated}.

To ensure computational efficiency, Haar-type filters are adopted when generating the scaling functions, which defines $\widehat{a}(x) = \cos(x/2)$ and $\widehat{b^{(1)}}(x) = \sin(x/2)$ for $x\in\R$. Alternatively, other types of filters, such as linear or quadratic filters, could also be considered as described in \cite{dong2017sparse}.

Suppose $l\in\sZ$ and $\uG\in\gV$, the \emph{undecimated framelets} $\boldsymbol{\varphi}_{l,p}(v)$ and $\boldsymbol{\psi}_{l,p}^r(v)$, $v\in\gV$ at scale $l$ are \emph{filtered Bessel kernels} (or summability kernels), which are constructed following
\begin{equation}\label{defn:ufra:ufrb}
\begin{aligned}
    \boldsymbol{\varphi}_{l,p}(v) &:= \sum_{\ell=1}^{N} \widehat{\alpha}\left(\frac{\lambda_{\ell}}{2^{l}}\right)
    \overline{\vu_{\ell}(p)}\vu_{\ell}(v), \\
    \boldsymbol{\psi}_{l,p}^r(v) &:= \sum_{\ell=1}^{N} \widehat{\beta^{(r)}}\left(\frac{\lambda_{\ell}}{2^{l}}\right)\overline{\vu_{\ell}(p)}\vu_{\ell}(v), \quad r = 1,\ldots,K.
\end{aligned}
\end{equation}

We say $\boldsymbol{\varphi}_{l,p}(v)$ and $\boldsymbol{\psi}_{l,p}^r(v)$ are with respect to the ``dilation'' at scale $l$ and the ``translation'' for the vertex $p\in\gV$. The construction of framelets are analogs of those of wavelets in $\R^d$. The functions $\alpha,\beta^{(r)}$ of $\Psi$ are named \emph{framelet generators} or \emph{scaling functions} for the undecimated framelet system.

\begin{theorem}[Equivalence Conditions of Framelet Tightness, \citep{zheng2022decimated}]\label{thm:UFS}
Let $\gG=(\gV,\gE,\mW)$ be a graph and $\{(\vu_{\ell},\lambda_{\ell})\}_{\ell=1}^N$ be a set of orthonormal eigenpairs for $l_2(\gG)$. Let $\Psi=\{\alpha;\beta^{(1)},\dots,\beta^{(K)}\}$ be a set of functions in $L_1(\R)$ with respect to a filter bank $\boldsymbol{\eta}=\{\maska;\maskb[1],\ldots,\maskb[K]\}$ that satisfies \eqref{eq:refinement}. An undecimated framelet system is denoted by $\ufrsys[J_1]^{J}(\Psi,\boldsymbol{\eta};\gG), J_1=1,\ldots, J$ ($J\geq1$) with framelets $\cfra$ and $\cfrb{r}$ in \eqref{defn:ufra:ufrb}. Then, the following statements are equivalent.
\begin{itemize}
\item[(i)] For each $J_1=1,\ldots,J$, the undecimated framelet system $\ufrsys[J_1]^{J}(\Psi,\boldsymbol{\eta};\gG)$ is a tight frame for $l_2(\gG)$, that is, $\forall f\in l_2(\gG)$,
    \begin{equation}\label{eq:f:UFS0}
    \|f\|^2 =\sum_{\uG\in V}\Big|\ipG{f,\boldsymbol{\varphi}_{J_1,p}}\Big|^2
    +\sum_{l=J_1}^{J}\sum_{r=1}^K\sum_{\uG\in V}\Big|\ipG{f,\boldsymbol{\psi}_{l,p}^r}\Big|^2.
    \end{equation}

\item[(ii)] For all $f\in l_2(\gG)$ and for $l=1,\ldots,J-1$, the following identities hold:
\begin{align}
&   f = \sum_{\uG\in V} \ipG{f,\cfra[J,\uG]}\cfra[J,\uG]
+\sum_{r=1}^K\sum_{\uG\in V}\ipG{f,\cfrb[J,\uG]{r}}\cfrb[J,\uG]{r},
\label{thmeq:normalization1}\\
& \sum_{\uG\in V} \ipG{f,\cfra[l+1,\uG]}\cfra[l+1,\uG]
= \sum_{\uG\in V} \ipG{f,\cfra}\cfra+
\sum_{r=1}^{K}\sum_{\uG\in V} \ipG{f,\cfrb{r}}\cfrb{r}. \label{thmeq:2scale1}
\end{align}

\item[(iii)] For all $f\in l_2(\gG)$ and for $l=1,\ldots, J-1$, the following identities hold:
\begin{align}
   & \|f\|^2 = \sum_{\uG\in V} \bigl|\ipG{f,\cfra[J,\uG]}\bigr|^{2}
   +\sum_{r=1}^K \sum_{\uG\in V}\bigl|\ipG{f,\cfrb[J,\uG]{r}}\bigr|^{2}, \quad\label{thmeq:normalization2}\\
    & \sum_{\uG\in V} \bigl|\ipG{f,\cfra[l+1,\uG]}\bigr|^{2}
    = \sum_{\uG\in V} \bigl|\ipG{f,\cfra}\bigr|^{2} + \sum_{r=1}^{K}\sum_{\uG\in V} \bigl|\ipG{f,\cfrb{r}}\bigr|^{2}.&\label{thmeq:2scale2}
\end{align}

\item[(iv)] The functions in $\Psi$ satisfy
\begin{align}
   1 = \left|\FT{\scala}\left(\frac{\eigvm}{2^{J}}\right)\right|^{2} + \sum_{r=1}^{K}\left|\FT{\scalb^{(r)}}\left(\frac{\eigvm}{2^{J}}\right)\right|^{2} &\quad \forall
  \ell=1,\ldots,\NV, \label{thmeq:nrm:alpha:beta}\\
     \left|\FT{\scala}\left(\frac{\eigvm}{2^{l+1}}\right)\right|^{2}
    = \left|\FT{\scala}\left(\frac{\eigvm}{2^{l}}\right)\right|^{2} + \sum_{r=1}^{K}\left|\FT{\scalb^{(r)}}\left(\frac{\eigvm}{2^{l}}\right)\right|^{2} &\quad \forall
 \begin{array}{l}
 \ell=1,\ldots,\NV,\\
 l=1,\ldots,J-1.
 \end{array}\label{thmeq:2scale:alpha:beta}
 \end{align}

\item[(v)] The identities in \eqref{thmeq:nrm:alpha:beta} hold and the filters in the filter bank $\boldsymbol{\eta}$ satisfy
\begin{align}
  \left|\FS{\maska}\left(\frac{\eigvm}{2^{l}}\right)\right|^{2} + \sum_{r=1}^{K} \left|\FS{\maskb}\left(\frac{\eigvm}{2^{l}}\right)\right|^{2} = 1 \quad \forall \ell\in\sigma_{\scala}^{(l)},\; l = 2,\ldots,J,
\label{thmeq:2scale:masks}
 \end{align}
with
\[
\sigma_{\scala}^{(l)}:=\left\{\ell\in\{1,\ldots,N\} : \FT{\scala}\left(\frac{\eigvm}{2^l}\right) \neq 0\right\}.
\]
\end{itemize}
\end{theorem}

\begin{remark}
    In this paper, we use $2$-norm.
\end{remark}

\section{Graph Framelet Message Passing}
\label{sec:fmp}
The aggregation of framelet coefficients for the neighborhood of the node $i$ takes over up to the $m$-th multi-hop $\mathcal{N}^{m}(i)$. 

\subsection{Graph Framelet Transforms}
In a spectral-based graph convolution, graph signals are transformed to a set of coefficients $\hat{\mX}=\boldsymbol{\gW} \mX^{\rm in}$ in frequency channels by the decomposition operator $\boldsymbol{\gW}$. The learnable filters are then trained for the spectral coefficients to approach node-level representative graph embeddings. 

This work implements undecimated framelet transforms \citep{zheng2021framelets} that generate a set of multi-scale and multi-level \emph{framelet coefficients} for the input graph signal, where the low-pass coefficients include general global trend of $\mX$, and the high-pass coefficients portray the local properties of the graph attributes at different degrees of detail. Consequently, conducting framelet transforms on a graph avoids trivially smoothing out rare patterns to preserve more energy for the graph representation and alleviate the oversmoothing issue during message aggregation. 


Recall that the orthonormal bases at different levels ($l$) and scales  ($r$) formulate the framelet decomposition operators $\boldsymbol{\gW}_{r,l}$, 
, which is applied to obtain the framelet coefficients $\hat{\mX}$. In particular, $\boldsymbol{\gW}_{0,J}$ is composed of the low-pass framelet basis $\boldsymbol{\varphi}_{J,p},\; p\in\gV$ for the low-pass coefficient matrix, \textit{i.e.}, $\hat{\mX}=\boldsymbol{\gW}_{0,J}\mX$ that approximates the global graph information. Meanwhile, the high-pass coefficient matrices $\boldsymbol{\gW}_{r,l}\mX$ with $r=1,\dots,K,\; l=1,\dots,J$ that record detailed local graph characteristics in different scale levels are decomposed by the associated high-pass framelet bases $\boldsymbol{\psi}_{l, p}^{(r)}$. Generally, framelet coefficients at larger scales contain more localized information with smaller energy.

As the framelet bases $\boldsymbol{\varphi}_{l,p},\boldsymbol{\psi}_{l, p}^{(r)}$ are defined by eigenpairs of the normalized graph Laplacian $\widetilde{\gL}$, the associated framelet coefficients can be recursively formulated by filter matrices. Specifically, denote the eigenvectors $\mU=[\vu_1,\dots,\vu_N]\in\R^{N\times N}$ and the eigenvalues $\Lambda=\operatorname{diag}(\lambda_1,\dots,\lambda_N)$ for the normalized graph Laplacian $\widetilde{\gL}$, for the low pass and the $r$th high pass,
\begin{equation} \label{eq:ft_initial}
\begin{aligned}
    &\boldsymbol{\gW}_{0,J}\mX=\mU \widehat{\alpha}\left(\frac{\Lambda}{2}\right) \mU^{\top} \mX , \\
    &\boldsymbol{\gW}_{r,l}\mX=\mU \widehat{\beta^{(r)}}\left(\frac{\Lambda}{2^{l+1}}\right) \mU^{\top} \mX 
    \quad \forall l=1,\dots,J.
\end{aligned}
\end{equation}

Alternatively, Chebyshev polynomials approximation is a valid solution to achieve efficient and scalable framelet decomposition \citep{dong2017sparse,zheng2021framelets} by avoiding time-consuming eigendecomposition. Let $m$ be the highest order of the Chebyshev polynomial involved. Denote the $m$-order approximation of $\alpha$ and $\{\beta^{(r)}\}_{r=1}^{K}$ by $\gT_0$ and $\{\gT_{r}\}_{r=1}^{K}$, respectively. The approximated framelet decomposition operator $\boldsymbol{\gW}^{\natural}_{r,l}$ (including $\boldsymbol{\gW}^{\natural}_{0,J}$) is defined as products of Chebyshev polynomials of the graph Laplacian $\gL$, \textit{i.e.}, 
\begin{equation*}
\boldsymbol{\gW}^{\natural}_{r,l}=
    \begin{cases}
    \gT_0\left(2^{-R} \gL\right), & l=1, \\[1mm]
    \gT_k\left(2^{R+l-1} \gL\right) \gT_0\left(2^{R+l-2} \gL\right) \dots \gT_0\left(2^{-R} \gL\right), & l=2,\dots,J.
    \end{cases}
\end{equation*}
Here the real-value dilation scale $R$ is the smallest integer such that $\lambda_{\max}= \lambda_{N}\leq 2^R\pi$. In this definition, it is required for the finest scale $1/2^{K+J}$ that guarantees $\lambda_{\ell}/2^{K+J-l}\in(0,\pi)$ for $\ell=1, 2, ..., N$.

Chebyshev polynomials approximation not only achieves efficient calculation to obtain framelet coefficients for the nodes, but also collects information from the nodes' neighbors in the framelet domain of the same channel. In particular, the aggregation of framelet coefficients for the neighborhood of the node $i$ takes over up to the $m$-hop $\mathcal{N}^{m}(i)$. Based on this, we propose the general framelet message passing framework.

\subsection{Framelet Message Passing}
We follow the general message passing scheme in Equation~\eqref{eq:mp} and define the vanilla \emph{Framelet message passing} (FMP) for the graph node feature $\mX^{(t)}$ at layer $t$ by 
\begin{equation} 
\begin{aligned}
\label{eq:fmp1}
    \mX_i^{(t)} &= \mX_i^{(t-1)}+\mZ_i^{(t)} \\
    \mZ_i^{(t)} &= \sigma \left(\sum_{j\in  \mathcal{N}^{m}(i)}\left(\sum_{r=1}^K\sum_{l=1}^J (\boldsymbol{\gW}^{\natural}_{r,l})_{i,j}\mX_j^{(t-1)}\Theta_r +(\boldsymbol{\gW}^{\natural}_{0,J})_{i,j}\mX_j^{(t-1)} \Theta_0 \right) \right),
\end{aligned}
\end{equation}
where  the propagated feature $\mZ_i^{(t)}$ at node $i$ sums over the low and high passes coefficients at all scales, $\sigma$ is the ReLU activation function, and $\Theta_r\in \mathbb{R}^{d\times d}$
are learnable parameter square matrices associated with the high passes and low pass, with the size $d$ equal to the number of features of $X^{(t-1)}$.

The framelet message passing has the matrix form 
\begin{equation} 
\begin{aligned}
\label{eq:fmp_mat}
    \mX^{(t)} &= \mX^{(t-1)}+\mZ^{(t)} \\
    \mZ^{(t)} &= \sigma \left(\sum_{r=1}^K\sum_{l=1}^J (\boldsymbol{\gW}^{\natural}_{r,l}) \mX^{(t-1)}\Theta_{r} +(\boldsymbol{\gW}^{\natural}_{0,J})\mX^{(t-1)}\Theta_{0}\right).
\end{aligned}
\end{equation}

As shown in the previous formulation, 
the aggregation takes over up to $m$-hop neighbours of the target node $i$. It exploits the framelet coefficients from an $m$-order approximated $\boldsymbol{\gW}^{\natural}_{r,l}$, which involves higher-order graph Laplacians, \textit{i.e.}, $\gL^m$, to make the spectral aggregation as powerful as conducting an $m$-hop neighborhood spatial MPNN \citep{gilmer2017neural}. Intuitively, the proposed update rule is thus more efficient than the conventional MPNN schemes in the sense that a single FMP layer is sufficient to reach distant nodes from $m$-hops away. Meanwhile, FMP splits framelet coefficients into different scales and levels instead of implementing a rough global summation, which further preserves the essential local and global patterns and circumvents the oversmoothing issue when stacking multiple convolutional layers.

\subsection{Continuous Framelet Message Passing with Neural ODE}
The vanilla FMP in Equation~\eqref{eq:fmp1} formulates a discrete version of spectral feature aggregation of node features. In order to gain extra expressivity for the node embedding, we design a continuous update scheme and formulate an enhanced FMP by neural ODE, such that
\begin{equation}
\begin{aligned}
\label{eq:fmp2}
    &{\partial \mX_i(t)}/{\partial t} = \mZ_i(t) \\
    &\mZ_i(t) = \sum_{j\in  \mathcal{N}^{m}(i)}\left(\sum_{r=1}^K\sum_{l=1}^J (\boldsymbol{\gW}^{\natural}_{r,l})_{i,j}\mX_j {(t)}\Theta_{r} +(\boldsymbol{\gW}^{\natural}_{0,J})_{i,j}\mX_j {(t)}\Theta_{0}\right) \\
    &\mX_i {(0)}=\text{MLP}(\mX_i) .
\end{aligned}
\end{equation}
    Here $\mX_i {(t)}$ denotes the encoded features of node $i$ at some timestamp $t$ during a continuous process.  The initial state is obtained by a simple MLP layer on the input node feature, \textit{i.e.}, $\mX_i {(0)}=\text{MLP}(\mX_i)$. The last embedding $\mX_i {(T)}$ can be obtained by numerically solving Equation~\eqref{eq:fmp2} with an ODE solver, which requires a stable and efficient numerical integrator. Since the proposed method is stable during the evolution within finite time, the vast explicit and implicit numerical methods are applicable with a considerable small step size $\Delta t$. In particular, we implement \textsc{Dormand–Prince5} (DOPRI5) \citep{dormand1980family} with adaptive step size for computational complexity. Note that the practical network depth (\textit{i.e.}, the number of propagation layers) is equal to the numerical iteration number that is specified in the ODE solver. 
    

\section{Limited Oversmoothing in Framelet Message Passing}
\label{sec:fmp_oversmoothing}
Following the discussion of oversmoothing in Section~\ref{sec:oversmoothing}, this section demonstrates that FMP provides a remedy for the oversmoothing predicament by decomposing $\mX$ into low-pass and high-pass coefficients with the energy conservation property.
\begin{proposition}[Energy conservation]
    \begin{equation}\label{eq:DE preserve}
    E(\mX) = E(\boldsymbol{\gW}_{0,J}\mX)+\sum_{r=1}^{K}\sum_{l=1}^{J} E(\boldsymbol{\gW}_{r,l}\mX),
\end{equation}
where $E(\boldsymbol{\gW}_{0,J}\mX)$ and $E(\boldsymbol{\gW}_{r,l}\mX)$ break down the total energy $E(\mX)$ into multi scales, and into low and high passes.
\end{proposition}

\begin{proof}
Recall the facts that  $E(\mX) = \tr(\mX^\top\widetilde{\gL}\mX)$ and  $\widetilde{\gL} = \mU \Lambda \mU^{\top}$, we have   
\begin{equation*}  
    E(\boldsymbol{\gW}_{0,J}\mX) = \tr\left(\mX^{\top}\boldsymbol{\gW}_{0,J}^{\top}\widetilde{\gL}\boldsymbol{\gW}_{0,J}\mX\right)= \tr\left(\mX^{\top}\mU \widehat{\alpha}\left(\frac{\Lambda}{2}\right)^2 \Lambda \mU^{\top} \mX \right) ,
\end{equation*}
and $\forall l=1,\dots,J, $ 
\begin{equation}
    E(\boldsymbol{\gW}_{r,l}\mX) = \tr\left(\mX^{\top}\boldsymbol{\gW}_{0,J}^{\top}\widetilde{\gL}\boldsymbol{\gW}_{r,l}\mX\right)= \tr\left(\mX^{\top}\mU \widehat{\beta^{(r)}}\left(\frac{\Lambda}{2^{l+1}}\right)^2 \Lambda \mU^{\top} \mX \right).
    \label{eqn:energy}
\end{equation}
By the identity \eqref{thmeq:2scale:alpha:beta}, we also have 
\begin{equation}
    \widehat{\alpha}\left(\frac{\Lambda}{2}\right)^2 + \sum_{r=1}^K\sum_{l=1}^J \widehat{\beta^{(r)}}\left(\frac{\Lambda}{2^{l+1}}\right)^2 = \mI.
    \label{eqn:partition2}
\end{equation}
Therefore, combining \eqref{eqn:energy} and \eqref{eqn:partition2}, we obtain that
\begin{equation*}
\begin{aligned}
    & E(\boldsymbol{\gW}_{0,J}\mX)+\sum_{r=1}^{K}\sum_{l=1}^{J} E(\boldsymbol{\gW}_{r,l}\mX) \\
    =& \tr\left(\mX^{\top}\mU \widehat{\alpha}\left(\frac{\Lambda}{2}\right)^2 \Lambda \mU^{\top} \mX \right) +  \sum_{r=1}^{K}\sum_{l=1}^{J} \tr\left(\mX^{\top}\mU \widehat{\beta^{(r)}}\left(\frac{\Lambda}{2^{l+1}}\right)^2 \Lambda \mU^{\top} \mX \right) \\
    =& \tr\left(\mX^{\top}\mU \left( \widehat{\alpha}\left(\frac{\Lambda}{2} \right)^2 + \sum_{r=1}^{K}\sum_{l=1}^{J}\widehat{\beta^{(r)}}\left(\frac{\Lambda}{2^{l+1}}\right)^2 \right) \Lambda \mU^{\top} \mX \right) \\
    =& \tr\left(\mX^{\top}\mU\Lambda\mU^{\top}\mX \right) \\
    =& E(\mX),
\end{aligned}
\end{equation*}
thus completing the proof.
\end{proof}

For the approximated framelet transforms $\boldsymbol{\gW}^{\natural}_{r,l}$, there is an energy change on RHS of \eqref{eq:DE preserve} due to the truncation error. The relatively smoother node features account for a lower level energy $E(\boldsymbol{\gW}_{0,J}\mX)$. In this way, we provide more flexibility to maneuver energy evolution as the energy does not decay when we use FMP in \eqref{eq:fmp1} or \eqref{eq:fmp2}. For simplicity, we focus on the exact framelet transforms. Under some mild assumptions on $\Theta_r$, we have the following two estimations on the Dirichlet energy. 

\begin{lemma}
    \citep{zhou2014some}
    \label{trace_inequality}
    Let $\mA_{i}\in\mathbb{R}^{d\times d}$ be positive semi-definite $(i=1,2,\cdots,n)$. Then,
    \begin{equation*}
\tr\left(\mA_{1}\mA_{2}\cdots\mA_{n}\right)\leq\sum_{i=1}^d\mu_{i}\left(\mA_{1}\mA_{2}\cdots\mA_{n}\right)\leq \tr\left(\mA_{1}\right)\tr\left(\mA_{2}\right)\cdots\tr\left(\mA_{n}\right),
    \end{equation*}
    where $\mu(\mA)$ denotes the singular value of matrix $\mA$.
\end{lemma}

The following theorem shows for the framelet message passing without activation function in the update, the Dirichlet energy of the feature at the $t$th layer $\mX^{(t)}$ is not less than that of the $(t-1)$th layer $\mX^{(t-1)}$. In fact, the $\mX^{(t)}$ and $\mX^{(t-1)}$ are equivalent up to some constant, and the constant with respect to the upper bound depends jointly on the level of the framelet decomposition and the learnable parameters' bound.
\begin{theorem}
    \label{Energy_Bound}
    Let the $\textsc{FMP}$ defined by
 \begin{equation} 
\begin{aligned}
    \mX^{(t)} &= \mX^{(t-1)}+\mZ^{(t)} \\
    \mZ^{(t)} &=\sum_{r=1}^K\sum_{l=1}^J (\boldsymbol{\gW}_{r,l}) \mX^{(t-1)}\Theta_{r} +(\boldsymbol{\gW}_{0,J})\mX^{(t-1)}\Theta_{0},
\end{aligned}
\end{equation}
where the parameter matrix $\Theta_{r}\in \mathbb{R}^{d\times d} \: (r=0,1,\cdots,K)$. Suppose $\Theta_{r}$ is positive semi-definite with $\tr\left(\Theta_{r}\right)\leq M$ for every $r$, then we can bound the Dirichlet energy of the graph node feature $\mX^{(t)}$ at layer $t$ by
\begin{equation}
    E\left(\mX^{(t-1))}\right)\leq E\left(\mX^{(t))}\right)\leq \left(M\sqrt{KJ+1}+1\right)^2 E\left(\mX^{(t-1)}\right).
\end{equation}
\end{theorem}

\begin{proof}
    By definition, we have
\begin{equation}
    \label{Dirichlet_Energy}
    \begin{aligned}
    E\left(\mX^{(t)}\right)&=\tr \left(\left(\mX^{(t)}\right)^{\top}\widetilde{\gL}\mX^{(t)} \right) \\
    &=\tr \left[\left(\mX^{(t-1)}+\mZ^{(t)}\right)^{\top}\widetilde{\gL}\left(\mX^{(t-1)}+\mZ^{(t)} \right)\right] \\
    &=E\left(\mX^{(t-1)}\right)+E\left(\mZ^{(t)}\right)+2\tr \left(\left(\mX^{(t-1)}\right)^{\top}\widetilde{\gL}\mZ^{(t)}\right).
    \end{aligned}
\end{equation}
By Lemma~\ref{trace_inequality}, we have
\begin{equation}
    \begin{aligned}
    &E\left(\mZ^{(t)}\right) \\
    &\leq M^2 \tr \left[ \left(\mX^{(t-1)}\right)^{\top}\mU \left(  \widehat{\alpha}\left(\frac{\Lambda}{2} \right)+\sum_{r=1}^K\sum_{l=1}^J \widehat{\beta^{r}}\left(\frac{\Lambda}{2^{l+1}}\right) \right)^{2}\Lambda\mU^{\top}\mX^{(t-1)} \right], \\
    &\tr \left(\left(\mX^{(t-1)}\right)^{\top}\widetilde{\gL}\mZ^{(t)}\right) \\
    &\leq M \tr \left[ \left(\mX^{(t-1)}\right)^{\top}\mU \left( \widehat{\alpha}\left(\frac{\Lambda}{2} \right)+\sum_{r=1}^K\sum_{l=1}^J \widehat{\beta^{r}}\left(\frac{\Lambda}{2^{l+1}}\right) \right)\Lambda\mU^{\top}\mX^{(t-1)} \right].
    \end{aligned}
\end{equation}
From the identity \eqref{thmeq:2scale:alpha:beta}, we know that for any eigenvalue $\lambda$ of the Laplacian $\widetilde{\gL}$,
\begin{equation}
\left|\widehat{\alpha}\left(\frac{\lambda}{2}\right)\right|^2 + \sum_{r=1}^K\sum_{l=1}^J \left|\widehat{\beta^{(r)}}\left(\frac{\lambda}{2^{l+1}}\right)\right|^2 = 1.
\end{equation}
Therefore, we have
\begin{equation}
\left(\widehat{\alpha}\left(\frac{\lambda}{2} \right)+\sum_{r=1}^K\sum_{l=1}^J \widehat{\beta^{r}}\left(\frac{\lambda}{2^{l+1}}\right)\right) ^{2} \leq KJ+1,
\end{equation}
which leads to
\begin{equation}
    \label{Dirichlet_Energy_Inequality}
    \begin{aligned}
    E\left(\mZ^{(t)}\right)\leq& M^2(KJ+1)E\left(\mX^{(t-1)}\right), \\
    \tr \left(\left(\mX^{(t-1)}\right)^{\top}\widetilde{\gL}\mZ^{(t)}\right)\leq& M\sqrt{KJ+1} \: E\left(\mX^{(t-1)}\right).
    \end{aligned}
\end{equation}
Combining \eqref{Dirichlet_Energy}, \eqref{Dirichlet_Energy_Inequality} and the facts that $E\left(\mZ^{(t)}\right)\geq 0$, $\tr \left(\left(\mX^{(t-1)}\right)^{\top}\widetilde{\gL}\mZ^{(t)}\right)\geq 0$ since $\Theta_r, \Theta_0$ are symmetric positive semi-definite matrices, then we have
\begin{equation*}
    E\left(\mX^{(t-1))}\right)\leq E\left(\mX^{(t))}\right)\leq \left(M\sqrt{KJ+1}+1\right)^2 E\left(\mX^{(t-1)}\right),
\end{equation*}
thus completing the proof.
\end{proof}

In Theorem \ref{Energy_Bound}, we only prove the case without nonlinear activation function $\sigma$.  For the \textsc{FMP}{\tiny ode} scheme in \ref{eq:fmp2}, we only apply a MLP layer on input features at the beginning. In this case, we prove the Dirichlet energy of the framelet message passing of any layer is non-decreasing and is equivalent to the initial time energy.
 \begin{theorem}
     Suppose the framelet message passing with ODE update scheme $\left(\textsc{FMP}_{\rm ode}\right)$ is defined by 
    \begin{equation}
\begin{aligned}
\label{eq:fmp_sim}
    &{\partial \mX(t)}/{\partial t} = \mZ(t), \\
    &\mZ(t) =  \sum_{r=1}^K\sum_{l=1}^J (\boldsymbol{\gW}_{r,l})\mX(t)\Theta_{r}   +(\boldsymbol{\gW}_{0,J}) \mX(t) \Theta_{0}, \\
    &\mX {(0)}={\rm MLP}(\mX) .
\end{aligned}
\end{equation}
Under the same assumption in Theorem~\ref{Energy_Bound}, the Dirichlet energy is bounded by
    \begin{equation*}
        E(\mX(0)) \leq E(\mX(t)) \leq e^{2M\sqrt{KJ+1}\: t} E(\mX(0)).
    \end{equation*}
 \end{theorem}

 \begin{proof}
     By definition, we have
     \begin{align*}
         \mX(t)=&\mX(t-\Delta t)+\Delta t\mZ(t-\Delta t), \\
         E\left(\mX(t)\right)=&E\left(\mX(t-\Delta t)\right)+2\Delta t\:\tr \left(\mX(t-\Delta t)^{\top}\widetilde{\gL}\mZ(t-\Delta t)\right) \\
         &+\Delta t^2 E\left(\mZ(t-\Delta t)\right), \\
         \frac{\mathrm{d}E(\mX(t))}{\mathrm{d}t}=&2\tr \left(\mX(t)^{\top}\widetilde{\gL}\mZ(t)\right).
     \end{align*}
     Then, by applying \eqref{Dirichlet_Energy_Inequality},
     \begin{equation*}
         0\leq \frac{\mathrm{d}E(\mX(t))}{\mathrm{d}t}\leq 2M\sqrt{KJ+1}\:E(\mX(t)),
     \end{equation*}
     which leads to
     \begin{equation*}
        E(\mX(0)) \leq E(\mX(t)) \leq e^{2M\sqrt{KJ+1}\: t} E(\mX(0)),
    \end{equation*}
    thus completing the proof.
\end{proof}

\section{Stability of Framelet Message Passing}
\label{sec:fmp_stability}
The stability of a GNN refers to its ability to maintain its performance when small changes are made to the input graph or to the model parameters. This section investigates how the multiscale property of graph framelets stabilizes the vanilla FMP in terms of fluctuation in the input node features.

\begin{lemma}\label{lem:framelet upper bound} For framelet transforms on graph $\mathcal{G}$ of level $l$ and $r$, and graph signal $X$ in $\ell_2(G)$,
\begin{equation*}                   \left\|\sum_{r=1}^K\boldsymbol{\gW}_{r,l}X\right\|\leq C \sqrt{\lambda_{\max}}\:2^{-\frac{l+1}{2}}\left\|X\right\|,
\end{equation*}
where $C$ is some constant and $\lambda_{\max}$ is the maximal eigenvalue of the graph Laplacian.
\end{lemma}
\begin{proof}
By the orthonormality of $\eigfm$,
\begin{equation*}
    \ipG{X,\cfra} = \sum_{\ell=1}^{\NV} \conj{\FT{\scala}\left(\frac{\eigvm}{2^{l}}\right)}\Fcoem{X}\:\eigfm(\uG), \quad
    \ipG{X,\cfrb{r}} = \sum_{\ell=1}^{\NV} \conj{\FT{\scalb^{(r)}}\left(\frac{\eigvm}{2^{l}}\right)}\Fcoem{X}\:\eigfm(\uG).
\end{equation*}

By \eqref{thmeq:2scale:alpha:beta}, 
\begin{align*}
     \sum_{\ell=1}^{\NV}\sum_{r=1}^{K} \left|\FT{\scalb^{(r)}}\left(\frac{\eigvm}{2^{l}}\right)\right|^2\left|\Fcoem{X}\right|^2\:\left\|\eigfm\right\|^2
    &=\sum_{\ell=1}^{\NV} \left[\FT{\scala}\left(\frac{\eigvm}{2^{l+1}}\right)^2-\FT{\scala}\left(\frac{\eigvm}{2^{l}}\right)^2\right]\bigl|\Fcoem{X}\bigr|^2\:\|\eigfm\|^2\\
    &=\frac{\lambda_\ell}{2^{l+1}}\sum_{\ell=1}^{\NV} \frac{\FT{\scala}\left(\frac{\eigvm}{2^{l+1}}\right)^2-\FT{\scala}\left(\frac{\eigvm}{2^{l}}\right)^2}{\frac{\lambda_\ell}{2^{l+1}}}\bigl|\Fcoem{X}\bigr|^2\\
    &\leq C^2 \lambda_{\max}2^{-(l+1)} \|X\|^2,
\end{align*}
where the inequality of the last line uses that the scaling function $\FT{\scala}$ is continuously differentiable on real axis, and Parseval's identity.

Then, by \eqref{eq:ft_initial},
\begin{equation*} 
\begin{aligned}
\left\|\sum_{r=1}^{K}\boldsymbol{\gW}_{r,l}X\right\|^2
&=\left\|\ipG{X,\cfrb[l,\cdot]{r}}\right\|^2 \\ 
&= \sum_{\ell=1}^{\NV}\sum_{r=1}^{K} \left|\FT{\scalb^{(r)}}\left(\frac{\eigvm}{2^{l}}\right)\right|^2\left|\Fcoem{X}\right|^2\:\left\|\eigfm\right\|^2 \\
&\leq C^2 \lambda_{\max}2^{-(l+1)} \|X\|^2,
\end{aligned}
\end{equation*}
thus completing the proof.
\end{proof}

\begin{theorem}[Stability]\label{thm:stability}  
The vanilla Framelet Message Passing (FMP) in \eqref{eq:fmp1} has bounded parameters: $\|\Theta_r\|\leq C_r$ with constants $C_r$ for $r=0,1,\dots,K$,
then the FMP is stable, \textit{i.e.}, there exists a constant $C$ such that
\begin{equation*}
    \left\|\mX^{(t)}-\widetilde{\mX}^{(t)}\right\|\leq C^t \left\|\mX^{(0)}-\widetilde{\mX}^{(0)}\right\|,
\end{equation*}
where $C:=1+4 C_1 \sqrt{\lambda_{\max}} \max_{r=0,1,\dots,K} C_r$, $\mX^{(0)}$ and $\widetilde{\mX}^{(0)}$ are the initial graph node features with $\mX^{(t+1)}$ and $\widetilde{\mX}^{(t+1)}$ at layer $t$.
\end{theorem}

\begin{proof}
Let
\begin{equation*}
    \mX^{(t)}=\mX^{(t-1)}+\mZ^{(t)},\quad
    \widetilde{\mX}^{(t)}=\widetilde{\mX}^{(t-1)}+\widetilde{\mZ}^{(t)},
\end{equation*}
with initialization $\mX^{(0)},\widetilde{\mX}^{(0)}\in \ell_2(\gG)$. It thus holds that
\begin{align*}
\left\|\mX^{(t)}-\widetilde{\mX}^{(t)}\right\|&\leq \left\|\mX^{(t-1)}-\widetilde{\mX}^{(t-1)}\right\|+\left\|\mZ^{(t)}-\widetilde{\mZ}^{(t)}\right\|, 
\end{align*}
where 
\begin{equation*}
\mZ^{(t)}=\sigma\left(\sum_{r=1}^K\sum_{l=1}^J
\boldsymbol{\gW}_{r,l}\mX^{(t-1)}\Theta_{r}+\boldsymbol{\gW}_{0,J}\mX^{(t-1)}\Theta_{0}\right)=:\sigma(Y^{(t)}).
\end{equation*}
By Lemma~\ref{lem:framelet upper bound},
\begin{align*}
&\left\|\mY^{(t)}-\widetilde{\mY}^{(t)}\right\| \\
\leq& \sum_{l=1}^J\left\|\sum_{r=1}^K\boldsymbol{\gW}_{r,l}\left(\mX^{(t-1)}-\widetilde{\mX}^{(t-1)}\right)\Theta_{r}\right\|
+\left\|\boldsymbol{\gW}_{0,J}\left(\mX^{(t-1)}-\widetilde{\mX}^{(t-1)}\right)\Theta_{0}\right\| \\
\leq& C_1 \left(\sum_{l=1}^J C_r \sqrt{\lambda_{\max}}2^{-\frac{l+1}{2}}\left\|\mX^{(t-1)}-\widetilde{\mX}^{(t-1)}\right\|+C_0\sqrt{\lambda_{\max}}\:\left\|\mX^{(t-1)}-\widetilde{\mX}^{(t-1)}\right\|\right)\\
\leq& 4 C_1 \sqrt{\lambda_{\max}} \max_{r=0,1,\dots,K} C_r \left\|\mX^{(t-1)}-\widetilde{\mX}^{(t-1)}\right\|.
\end{align*}
Therefore, 
\begin{align*}
\left\|\mZ^{(t)}-\widetilde{\mZ}^{(t)}\right\| &=\left\|\sigma\left(\mY^{(t)}-\widetilde{\mY}^{(t)}\right)\right\| \\
&\leq \left\|\mY^{(t)}-\widetilde{\mY}^{(t)}\right\| \\
&\leq 4 C_1 \sqrt{\lambda_{\max}} \max_{r=0,1,\dots,K} C_r \left\|\mX^{(t-1)}-\widetilde{\mX}^{(t-1)}\right\|, 
\end{align*}
which leads to 
\begin{equation*}
    \left\|\mX^{(t)}-\widetilde{\mX}^{(t)}\right\|\leq C\left\|\mX^{(t-1)}-\widetilde{\mX}^{(t-1)}\right\|\leq C^t \left\|\mX^{(0)}-\widetilde{\mX}^{(0)}\right\|,
\end{equation*}
thus completing the proof.
\end{proof}
\begin{remark}
    In the neural ODE case, we can prove  
    \begin{equation}
        \left\|\mX^{(t+1)}-\widetilde{\mX}^{(t+1)}\right\|\leq e^{Ct}\left\|\mX^{(0)}-\widetilde{\mX}^{(0)}\right\|,
    \end{equation}
    where $C$ is a constant. Given the number of layers, the ${\rm FMP}_{\rm ode}$ is stable. 
\end{remark}

\section{Numerical Analysis} 
\label{sec:exp}
This section validates the performance of FMP with node classification on a diverse of popular benchmark datasets, including 4 homogeneous graphs and 3 heterogeneous graphs. The models are programmed with \texttt{PyTorch-Geometric} (version 2.0.1) and \texttt{PyTorch} (version 1.7.0) and tested on NVIDIA\textsuperscript{\textregistered} Tesla V100 GPU with 5,120 CUDA cores and 16GB HBM2 mounted on an HPC cluster. 

\subsection{Experimental Protocol}
\paragraph{Datasets}
We evaluate the performance of our proposed FMP by the three most widely used citation networks \citep{yang2016revisiting}: \textbf{Cora}, \textbf{Citeseer}, and \textbf{Pubmed}. We also adopt three heterogeneous graphs (\textbf{Texas}, \textbf{Wisconsin}, and \textbf{Cornell}) from \textbf{WebKB} dataset \citep{garcia2016using} that record the web pages from computer science departments of different universities and their mutual links.

\paragraph{Baselines}
We compare the performance of FMP with various state-of-the-art baselines, which are implemented based on \texttt{PyTorch-Geometric} or their open
repositories. On both homogeneous and heterogeneous graphs, a set of powerful classic shallow GNN models are compared against, including MLP, GCN \citep{kipf2016semi}, GAT \citep{velivckovic2017graph}, and \textsc{GraphSage} \citep{hamilton2017inductive}. We also look into \textsc{JKNet} \citep{xu2018representation}, SGC \citep{wu2019simplifying}, and APPNP \citep{gasteiger2018predict} which are specially designed for alleviating the over-smoothing issue of graph representation learning. Such techniques usually introduce a regularizer or bias term to increase the depth of graph convolutional layers. In order to validate the effective design of the neural ODE scheme, the performance of three other continuous GNN layers, including GRAND \citep{chamberlain2021grand}, CGNN \citep{xhonneux2020continuous}, and GDE \citep{poli2019graph}, are investigated on homogeneous graph representation learning tasks. Furthermore, other specially designed convolution mechanisms, \textit{i.e.,} GCNII \citep{chen2020simple}, H{\tiny 2}GCN \citep{zhu2020beyond}, and ASGAT \citep{li2021beyond}, are compared for the learning tasks on the three heterogeneous graphs.


\begin{table}[t]
\caption{Hyperparameter searching space for node classification tasks.}
\label{tab:searchSpace}
\begin{center}
\resizebox{0.7\textwidth}{!}{
    \begin{tabular}{lll}
    \toprule
    \textbf{Hyperparameters} & \textbf{Search Space}  & \textbf{Distribution} \\
    \midrule
    learning rate & $[10^{-3},10^{-2}]$ & log-uniform \\
    weight decay & $[10^{-3},10^{-1}]$ & log-uniform \\
    dropout rate & $[0.0,0.8]$ & uniform \\
    hidden dim & $\{64,128,256\}$ & categorical \\
    layer & $[1,10]$ & uniform \\
    optimizer & \{\textsc{Adam}, \textsc{Adamax}\} & categorical \\
    \bottomrule
    \end{tabular}
}
\end{center}
\end{table}

\paragraph{Setup}
We construct FMP with two convolutional layers for learning node embeddings, the output of which is proceeded by a softmax activation for final prediction. The aggregator $\gamma$ is a $2$-layer MLP for \textsc{FMP}{\tiny mlp} in \eqref{eq:fmp1}, and a linear projection for \textsc{FMP}{\tiny ode} in \eqref{eq:fmp2}. Grid search is conducted to fine-tune the key hyperparameters from a considerably small range of search space, see Table~\ref{tab:searchSpace}. 
All the average test accuracy and the associated standard deviation come from $10$ runs.

\begin{table*}[t]
\caption{Average accuracy of node classification on homogeneous graphs over $10$ repetitions. The top three are highlighted by \textbf{\textcolor{red}{First}}, \textbf{\textcolor{violet}{Second}}, \textbf{Third}.}
\label{tab:node_classification_homo}
\begin{center}
\resizebox{0.85\textwidth}{!}{
    \begin{tabular}{lccc}
    \toprule
    & \textbf{Cora} & \textbf{Citeseer} & \textbf{Pubmed} \\ 
    homophily level & 0.83 & 0.71 & 0.79  \\
    \midrule
    MLP & $57.8${\scriptsize $\pm0.1$} & $61.2${\scriptsize $\pm0.1$} & $73.2${\scriptsize $\pm0.1$} \\
    GCN \citep{kipf2016semi} & $82.4${\scriptsize $\pm0.3$} & $70.7${\scriptsize $\pm0.4$} & $79.4${\scriptsize $\pm0.2$} \\
    GAT \citep{velivckovic2017graph} & $82.5${\scriptsize $\pm0.3$} & $72.1${\scriptsize $\pm0.4$} & $79.1${\scriptsize $\pm0.2$} \\
    \textsc{GraphSage} \citep{hamilton2017inductive} & $82.1${\scriptsize $\pm0.3$} & $71.8${\scriptsize $\pm0.4$} & $79.2${\scriptsize $\pm0.3$} \\
    \midrule
    SGC \citep{wu2019simplifying} & $81.9${\scriptsize $\pm0.2$} & \bm{$72.2${\scriptsize $\pm0.2$}} & $78.3${\scriptsize $\pm0.1$} \\
    \textsc{JKNet} \citep{xu2018representation} & $81.2${\scriptsize $\pm0.5$} & $70.7${\scriptsize $\pm0.9$} & $78.6${\scriptsize $\pm0.3$} \\
    APPNP \citep{gasteiger2018predict} & $83.3${\scriptsize $\pm0.4$} & $71.7${\scriptsize $\pm0.5$} & \textcolor{red}{\bm{$80.1${\scriptsize $\pm0.3$}}} \\
    \midrule
    GRAND \citep{chamberlain2021grand} & \bm{$83.6${\scriptsize $\pm0.5$}} & $70.8${\scriptsize $\pm1.1$} & \textcolor{violet}{\bm{$79.7${\scriptsize $\pm0.3$}}} \\
    CGNN \citep{xhonneux2020continuous} & 
    $81.7${\scriptsize $\pm0.7$} & $68.1${\scriptsize $\pm1.2$} & $80.2${\scriptsize $\pm0.3$} \\
    GDE \citep{poli2019graph} & \textcolor{red}{\bm{$83.8${\scriptsize $\pm0.5$}}} & \textcolor{violet}{\bm{$72.5${\scriptsize $\pm0.5$}}} & $79.5${\scriptsize $\pm0.3$} \\
    \midrule
    \textsc{FMP}{\tiny mlp} (ours) & $82.2${\scriptsize $\pm0.70$} & $71.7${\scriptsize $\pm0.40$} &$79.4${\scriptsize $\pm0.30$} \\
    \textsc{FMP}{\tiny ode} (ours) & \textcolor{violet}{\bm{$83.6${\scriptsize $\pm0.8$}}} & \textcolor{red}{\bm{$72.7${\scriptsize $\pm0.9$}}} & \bm{$79.6$}{\scriptsize $\pm0.9$} \\
    \bottomrule\\[-2.5mm]
    \end{tabular}}
\end{center}
\end{table*}

\subsection{Node Classification}
The proposed FMP is validated on six graphs to conduct node classification tasks. We call the three datasets that have relatively high \emph{homophily level} homogeneous graphs and the other three heterogeneous graphs. To be specific, we follow \cite{pei2019geom} and define the homophily levels of a graph by the overall degree of the consistency among neighboring nodes' labels, \textit{i.e.,}
\begin{equation*}
    \gH = \frac1{|V|}\sum_{v\in\gV}\frac{\text{\# $v$'s neighbors that have identical label as $v$}}{\text{\# $v$'s neighbors}}.
\end{equation*}

\vspace{2mm}
\paragraph{Homogeneous Graphs}
Table~\ref{tab:node_classification_homo} reports the performance of three node classification tasks on the three citation networks. The baseline models' prediction scores are retrieved from previous literature. In particular, results on the first seven (classic models and oversmoothing-surpassed models) are provided in \cite{zhu2021interpreting}, and the last three (continuous methods) are obtained from \cite{chamberlain2021grand}. \textsc{FMP}{\tiny ode} achieves top performance over its competitors on all three tasks and \textsc{FMP}{\tiny mlp} obtains a comparable prediction accuracy. Both \textsc{FMP} variants take multiple hops' information into account in a single graph convolution, while the majority of left methods require propagating multiple times to achieve the same level of the receptive field. However, during the aggregation procedure, information from long-range neighbors is dilated progressively, and the features from the closest neighbors (\textit{e.g.,} one-hop neighbors) are augmented recklessly. Consequently, \textsc{FMP}{\tiny ode} gains a stronger learning ability with respect to other models on the entire graph. Moreover, \textsc{FMP}{\tiny ode} utilizes a continuous update scheme and reaches deeper network architectures to pursue even more expressive propagation.

\begin{table*}[t]
\caption{Average accuracy of node classification on heterogeneous graphs over $10$ repetitions. \textbf{\textcolor{red}{First}}, \textbf{\textcolor{violet}{Second}}, \textbf{Third}.}
\label{tab:node_classification_hetero}
\begin{center}
\resizebox{0.85\textwidth}{!}{
    \begin{tabular}{lccc}
    \toprule
    & \textbf{Texas} & \textbf{Wisconsin} & \textbf{Cornell} \\ 
    homophily level & 0.11 & 0.21 & 0.30 \\
    \midrule
    MLP & $81.9${\scriptsize $\pm4.8$} & \bm{$85.3${\scriptsize $\pm3.6$}} & $81.1${\scriptsize $\pm6.4$} \\
    GCN \citep{kipf2016semi} & $59.5${\scriptsize $\pm5.3$} & $59.8${\scriptsize $\pm7.0$} & $57.0${\scriptsize $\pm4.8$} \\
    GAT \citep{velivckovic2017graph} & $59.5${\scriptsize $\pm5.3$} & $59.8${\scriptsize $\pm7.0$} & $58.9${\scriptsize $\pm3.3$} \\
    \textsc{GraphSage} \citep{hamilton2017inductive} & $82.4${\scriptsize $\pm6.1$} & $81.2${\scriptsize $\pm5.6$} & $76.0${\scriptsize $\pm5.0$} \\
    \midrule
    GCN-\textsc{JKNet} \citep{xu2018representation} & $66.5${\scriptsize $\pm6.6$} & $74.3${\scriptsize $\pm6.4$} & $64.6${\scriptsize $\pm8.7$} \\
    APPNP \citep{gasteiger2018predict} & $60.3${\scriptsize $\pm4.3$} & $48.4${\scriptsize $\pm6.1$} & $58.9${\scriptsize $\pm3.2$} \\
    \midrule
    GCNII \citep{chen2020simple} & $76.5$ & $77.8$ & $77.8$ \\
    H{\tiny 2}GCN \cite{zhu2020beyond} & \textcolor{violet}{\bm{$84.9${\scriptsize $\pm6.8$}}} & \textcolor{violet}{\bm{$86.7${\scriptsize $\pm4.9$}}} & \bm{$82.2${\scriptsize $\pm4.8$}} \\
    ASGAT \citep{li2021beyond} & \bm{$84.6${\scriptsize $\pm5.8$}} & $82.2${\scriptsize $\pm3.2$} & \textcolor{violet}{\bm{$86.9${\scriptsize $\pm4.2$}}} \\
    \midrule
    \textsc{FMP}{\tiny ode} (ours) & \textcolor{red}{\bm{$87.6$}{\scriptsize $\pm5.2$}} & \textcolor{red}{\bm{$93.7$}{\scriptsize $\pm1.2$}} & \textcolor{red}{\bm{$93.8$}{\scriptsize $\pm1.5$}} \\
    \bottomrule\\[-2.5mm]
    \end{tabular}}
\end{center}
\end{table*}

\paragraph{Heterogeneous Graphs}
As \textsc{FMP}{\tiny ode} can access multi-hop neighbors for a central node in one shot, we next explore its capability in distinguishing dissimilar neighboring nodes by prediction tasks on three heterogeneous datasets. The prediction accuracy for node classification tasks can be found in Table~\ref{tab:node_classification_hetero}. The results of the first six baseline methods (for general graphs) are acquired from \cite{zhu2020beyond}, and the scores on the last three baselines (that are specifically designed for heterogeneous graphs) are provided by the associated original papers. \textsc{FMP}{\tiny ode} outperforms the second best model noticeably by $3\%-8\%$. Furthermore, \textsc{FMP}{\tiny ode}'s performance is significantly more stable than other methods with smaller standard deviations on repetitive runs. For clarification, all the baselines' results are the same as ours that takes the average prediction accuracy over $10$ repetitions. The two exceptions are GCNII and ASGAT, which takes the results from $1$ and $3$ runs, respectively.



\subsection{Dirichlet Energy}
\begin{figure}[!tp]
    \centering
    \includegraphics[width=0.7\linewidth]{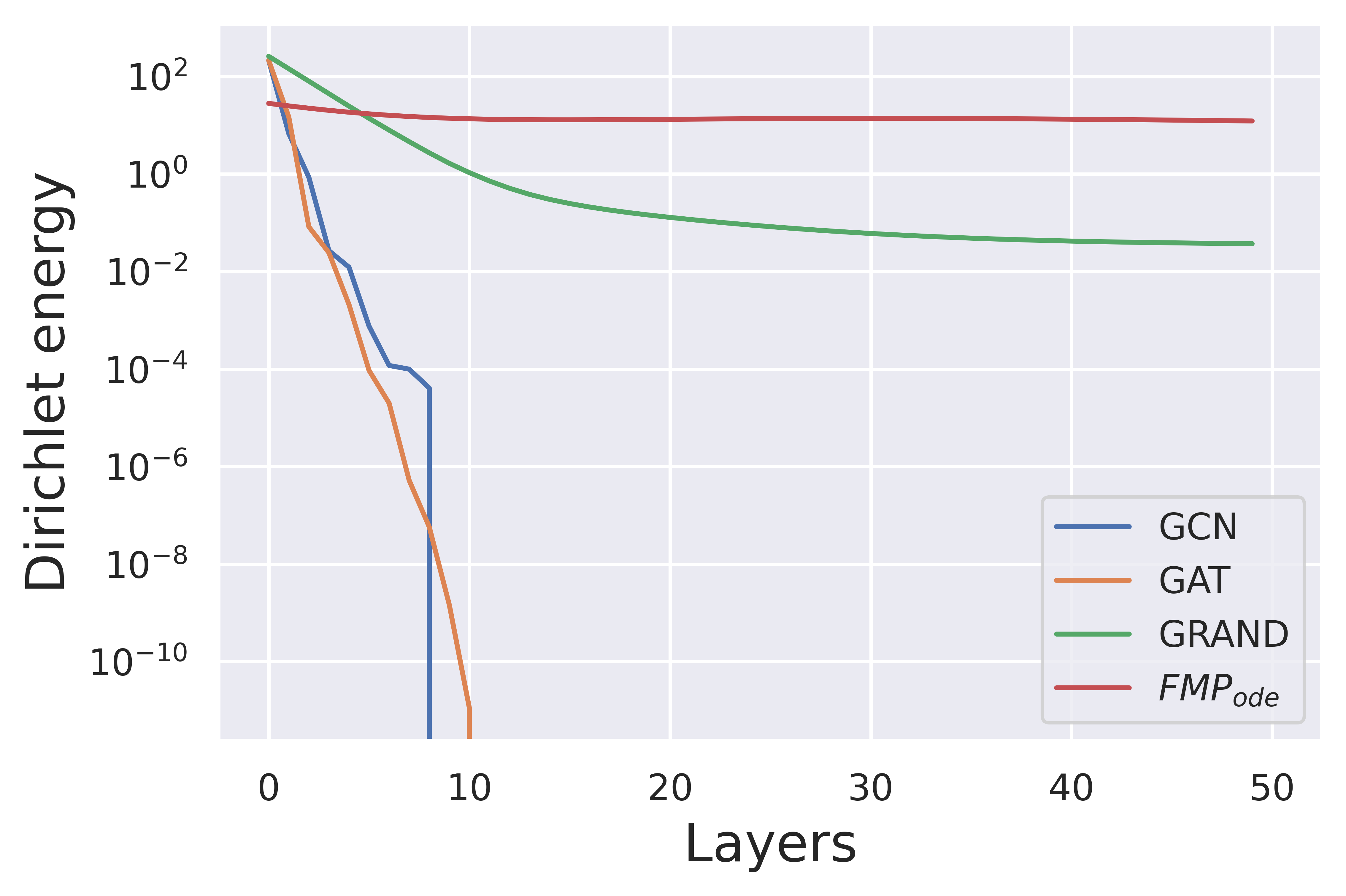}
    \caption{Energy evolution for $\textsc{FMP}{\tiny \text{ode}}$.}
    \label{fig:energy}
\end{figure}

We first illustrate the evolution of the Dirichlet energy of FMP by an undirected synthetic random graph. The synthetic graph has 100 nodes with two classes and 2D feature which is sampled from the normal distribution with the same standard deviation $\sigma = 2$ and two means $\mu_1 = -0.5,$ $\mu_2 = 0.5.$ The nodes are connected randomly with probability $p = 0.9$ if they are in the same class, otherwise nodes in different classes are connected with probability $p = 0.1.$ We compare the energy behavior of GNN models with four message passing propagators: GCNs \citep{kipf2016semi}, GAT \citep{velivckovic2017graph}, GRAND \citep{chamberlain2021grand} and $\textsc{FMP}{\tiny \text{ode}}$. We visualize how the node features evolve during $50$ layers of message passing process, from input features at layer $0$ to output features at layer $50$. For each model, the parameters are all properly initialized. The Dirichlet energy of each layer's output are plotted in logarithm scales. Traditional GNNs such as GCN and GAT suffer oversmoothing as the Dirichlet energy exponentially decays to zero within the first ten layers. GRAND relieves this problem by adding skip connections. $\textsc{FMP}{\tiny \text{ode}}$   increases energy mildly over network propagation. The oversmoothing issue in GNNs is circumvented with $\textsc{FMP}{\tiny \text{ode}}$. 


\section{Related Work}
\label{sec:review}
\subsection{Message Passing on Graph Neural Networks}
Message Passing Neural Network (MPNN) \citep{gilmer2017neural} establishes a general computational framework of graph feature propagation that covers the majority of update rules for attributed graphs. In each round, every node computes a message and passes the message to its adjacent nodes. Next, a random
node aggregates the messages it receives and uses the aggregation to update its embedding. Different graph convolutions vary in the choice of aggregation. For instance, GCN \citep{kipf2016semi} and \textsc{GraphSage} \cite{hamilton2017inductive} operate (selective) summation to neighborhood features, and some other work refines the aggregation weights by the attention mechanism \citep{xie2020mgat,brody2022how} or graph rewiring \citep{ruiz2020graphon,bruel2022rewiring,banerjee2022oversquashing,deac2022expander}. While constructing propagation rules 
from the adjacency matrix, \textit{i.e.}, spatial-based graph convolution, is effective enough to encode those relatively-simple graphs instances, it has been demonstrated that such methods ignore high-frequency local information in the input graph signals \citep{bo2021beyond}. With an increased number of layers, such convolutions only learn node degree and connected components under the influence of the Laplacian spectrum \citep{oono2019graph}, and the non-linear operation merely slows down the convolution speed \citep{wu2019simplifying}.


\subsection{Spectral Graph Transforms}
Spectral-based graph convolutions have shown promising performance in transferring a trained graph convolution between different graphs, \textit{i.e.}, the model is transferable and generalizable \citep{levie2019transferability,gama2020stability}. The output of spectral-based methods is stable with respect to perturbations of the input graphs \citep{ruiz2021graph,zhou2022robust,maskey2022generalization}. In literature, a diverse set of spectral transforms have been applied on graphs, such as Haar \citep{li2020fast,wang2020haar} where Haar convolution and Haar pooling were proposed using the hierarchical Haar bases on a chain of graphs, scattering \citep{gao2019geometric,ioannidis2020pruned}where the contractive graph scattering wavelets mimic the deep neural networks with wavelets are neurons and the decomposition is the propagation of the layers, needlets \citep{yi2022approximate} when the semidiscrete spherical wavelets are used to define approximate equivariance neural networks for sphere data, 
and framelets \citep{dong2017sparse,wang2019tight,zheng2022decimated} where the spectral graph convolutions are induced by graph framelets. 

\subsection{Oversmoothness in Graph Representation}
As many spatial-based graph convolutions merely perform a low-pass filter that smooths out the local perturbations in the input graph signal, smoothing has been identified as a common phenomenon in MPNNs \citep{li2018deeper,zhao2019pairnorm,nt2019revisiting}. Previous works usually measure and quantify the level of oversmoothing in a graph representation by the distances between node pairs \citep{rong2019dropedge,zhao2019pairnorm,chen2020measuring,Hou2020MeasuringAI}. When carrying out multiple times of signal smoothing operations, the output of nodes from different clusters tends to converge to similar vectors. 
While oversmoothing deteriorates the performance in GNNs, efforts have been made to preserve the identity of individual messages by modifying the message passing scheme, such as introducing jump connections \citep{xu2018representation,chen2020simple}, sampling neighboring nodes and edges \citep{rong2019dropedge,feng2020graph}, adding regularizations \citep{chen2020measuring,zhou2020towards,yang2021graph}, and increasing the complexity of convolutional layers \citep{balcilar2021breaking,geerts2021let,bodnar2021weisfeiler,wang2022acmp}.
Other methods try to trade-off graph smoothness with the fitness of the encoded features \citep{zhu2021interpreting,zhou2021admm,zhou2022robust} or postpone the occurrence of oversmoothing by mechanisms, such as residual networks \citep{li2021training,liu2021graph} and the diffusion scheme \citep{chamberlain2021grand,zhao2021adaptive}.

\subsection{Stability of Graph Convolutions}
While MPNN-based spatial graph convolutions have attracted increasing attentions due to their intuitive architecture and promising performance. However, when the generalization progress is required from a small graph to a larger one, summation-based MPNNs do not show satisfying stability and transferability \citep{yehudai2021local}. In fact, the associated generalization bound from one graph to another is directly proportional to the largest eigenvalue of the graph Laplacian \citep{verma2019stability}. Consequently, employing spectral graph filters that are robust to certain structural perturbations becomes a necessary condition for transferability in graph representation learning. 
Earlier work discussed the linear stability of spectral graph filters in the Cayley smoothness space with respect to the change in the normalized graph Laplacian matrix \citep{levie2019transferability,kenlay2020stability,gama2020stability}. The stability of graph-based models could be measured by the statistical properties of the model, where the graph topology and signal are viewed as random variables. It has been shown that the output of the spectral filters in stochastic time-evolving graphs behave the same as the deterministic filters in expectation \citep{isufi2017filtering}. \citet{Ceci2018RobustGS} approximated the original filtering process with uncertainties in the graph topology. Later on, stochastic analysis is leveraged to learn graph filters with topological stochasticity \citep{gao2021stochastic}. \citet{maskey2022generalization} proved the stability of the spatial-based message passing. \citet{kenlay2020stability,kenlay2021stability,kenlay2021interpretable} proved the stability for the spectral GCNs.

\section{Conclusion}
\label{sec:conclusion}
This work proposes an expressive framelet message passing for GNN propagation. The framelet coefficients of neighboring nodes provide a graph rewiring scheme to amalgamate features in the framelet domain. We show that our FMP circumvents oversmoothing that appears in most spatial GNN methods. The spectral information reserves extra expressivity to the graph representation by taking multiscale framelet representation into account. Moreover, FMP has good stability in learning node feature representations at low computational complexity.







\vskip 0.2in
\bibliography{reference.bib}

\end{document}

%% file: math_commands.tex
\usepackage{amsmath,amsfonts,bm}

\newcommand{\Fcoem}[2][\ell]{ 
    \widehat{#2}_{#1}
}
\newcommand{\uG}{p}

\newcommand{\cfra}[1][l,\uG]{ 
    \boldsymbol{\varphi}_{#1}
    }
\newcommand{\cfrb}[2][l,\uG]{ 
    \boldsymbol{\psi}_{#1}^{#2}
}
\newcommand{\scala}{ 
    \alpha
}
\newcommand{\scalb}{ 
    \beta
}
\newcommand{\eigvm}[1][\ell]{ 
    \lambda_{#1}
}

\newcommand{\eigfm}[1][\ell]{ 
    \boldsymbol{u}_{#1}
}
\newcommand{\NV}{ 
    N
}

\newcommand{\ipG}[1]{\left\langle #1 \right\rangle
}
\newcommand{\mask}{ 
    h
}
\newcommand{\maska}{ 
    a
}
\newcommand{\maskb}[1][r]{ 
    b^{(#1)}
}
\newcommand{\FT}[1]{ 
    \widehat{#1}
}

\newcommand{\FS}[1]{ 
    \widehat{#1}
}
\newcommand{\ufrsys}[1][J_0]{ 
    \mathsf{UFS}_{#1}
}

\def\eqref#1{(\ref{#1})}

\def\1{\bm{1}}

\def\vf{{\bm{f}}}
\def\vg{{\bm{g}}}

\def\vu{{\bm{u}}}

\def\mA{{\bm{A}}}

\def\mD{{\bm{D}}}

\def\mH{{\bm{H}}}
\def\mI{{\bm{I}}}

\def\mP{{\bm{P}}}

\def\mU{{\bm{U}}}

\def\mW{{\bm{W}}}
\def\mX{{\bm{X}}}
\def\mY{{\bm{Y}}}
\def\mZ{{\bm{Z}}}

\DeclareMathAlphabet{\mathsfit}{\encodingdefault}{\sfdefault}{m}{sl}
\SetMathAlphabet{\mathsfit}{bold}{\encodingdefault}{\sfdefault}{bx}{n}

\def\gE{{\mathcal{E}}}

\def\gG{{\mathcal{G}}}
\def\gH{{\mathcal{H}}}

\def\gL{{\mathcal{L}}}

\def\gN{{\mathcal{N}}}

\def\gT{{\mathcal{T}}}

\def\gV{{\mathcal{V}}}
\def\gW{{\mathcal{W}}}

\def\sC{{\mathbb{C}}}

\def\sZ{{\mathbb{Z}}}

\newcommand{\tr}{\mathrm{tr}}

\newcommand{\R}{\mathbb{R}}

\newcommand{\conj}[1]{ 
 \overline{#1}
}
